\documentclass[twoside]{article}
\usepackage{nips12submit_e,amsthm,times}

\usepackage{verbatim} 
\usepackage{xcolor,times}
\usepackage{paralist}
\usepackage{epsfig}
\usepackage{graphicx}
\usepackage{graphics}
\usepackage[vlined,algoruled,linesnumbered]{algorithm2e}
\usepackage[colorlinks,linkcolor=red,citecolor=blue,urlcolor=blue]{hyperref}

\input{Definitions}
\usepackage[compact]{titlesec}

\titlespacing{\section}{0pt}{-.3ex}{-.5ex}
\titlespacing{\subsection}{-1pt}{-1ex}{-1ex}
\titlespacing{\subsubsection}{-1pt}{-1ex}{-1ex}

\newcommand{\as}{\setlength{\abovedisplayskip}{1pt}}
\newcommand{\bs}{\setlength{\belowdisplayskip}{1pt}}

\title{The Interplay Between Stability and Regret \\ in Online Learning}
\author{
Ankan Saha
\\
Department of Computer Science\\
University of Chicago\\
\texttt{ankans@cs.uchicago.edu} \\
\And
Prateek Jain \\
Microsoft Research India\\
\texttt{prajain@microsoft.com} \\
\And
Ambuj Tewari \\
Department of Statistics\\
University of Michigan\\
\texttt{tewaria@umich.edu}
}

\allowdisplaybreaks

\begin{document}
\maketitle
\begin{abstract}
  This paper considers the stability of online learning algorithms and
  its implications for learnability (bounded regret).  We introduce a
  novel quantity called {\em forward regret} that intuitively measures
  how good an online learning algorithm is if it is allowed a one-step
  look-ahead into the future.  We show that given stability, bounded
  forward regret is equivalent to bounded regret. We also show that
  the existence of an algorithm with bounded regret implies the
  existence of a stable algorithm with bounded regret and bounded
  forward regret. The equivalence results apply to general, possibly
  non-convex problems. To the best of our knowledge, our analysis
  provides the first general connection between stability and regret
  in the online setting that is not restricted to a particular class
  of algorithms.  Our stability-regret connection provides a simple
  recipe for analyzing regret incurred by any online learning
  algorithm. Using our framework, we analyze several existing online
  learning algorithms as well as the ``approximate'' versions of
  algorithms like RDA that solve an optimization problem at each
  iteration. Our proofs are simpler than existing analysis for the
  respective algorithms, show a clear trade-off between stability and
  forward regret, and provide tighter regret bounds in some cases. 
  Furthermore, using our recipe, we
   analyze ``approximate'' versions of several algorithms such as
   follow-the-regularized-leader (FTRL) that requires solving an
   optimization problem at each step.
\end{abstract}


\vspace{-10pt}
\section{Introduction}
\label{sec:intro}

The fundamental role of stability in determining the generalization
ability of learning algorithms in the setting of iid data is now well
recognized. Moreover, our knowledge of the connection between
stability and generalization is beginning to achieve a fair degree of
maturity (see, for instance,
\cite{BouEli00,KutinN02,RakhlinMP05,ShaShaSreSri10}).  However, the
same cannot be said regarding our understanding of the role of
stability in online adversarial learning.

Recently, several results have shown connections between learnability
of a concept class and stability of its empirical risk minimizer
(ERM). Apart from theoretical interest, such insights into stability
and learnability, can potentially help in designing more practical
algorithms. For example, \cite{KutinN02} show that under certain
settings, stability is a more general characterization than
VC-dimension; good generalization performance can be guaranteed for
concept classes with stable ERM, even if its VC-dimension is infinite.

However, most of the existing implications of stability are in the
batch or i.i.d. learning setting, with only a few results in the online
adversarial setting. 
Online learning can be modeled as a sequential two-player game between
a player (learner) and an adversary where, at each step, the player
takes an action from a set and the adversary plays a loss function.
The player's loss is evaluated by applying the adversary's move to the
player's action and key quantity to control is the \emph{regret} of
the player in hindsight. Understanding stability in the online
learning setting is not only a challenging theoretical problem but is
also important from the point of view of applications. For instance,
stability allows us to derive guarantees that apply to dependent
(non-iid) data \cite{AgarwalD11} and is critical in areas such as
privacy \cite{JainKT11}.

There is a fundamental challenge in extending the connection between
stability and learnability from the iid to the online case. In the iid
setting, empirical risk minimization (ERM) serves as a canonical
learning algorithm \cite{Vap98}. Thus, given \emph{any} hypothesis
class, it is sufficient to just analyze the stability of ERM over the
class to characterize its learnability in the batch
setting. Unfortunately, no such canonical scheme is known for online
learning, making it significantly more involved to forge connections
between online learnability and stability.  We circumvent this
difficulty by studying connections between stability and regret of
\emph{arbitrary} online learning {\em algorithms}.

In this paper, we circumvent the above mentioned issue by studying
connections between stability and regret of 
learning {\em algorithms}, rather than online learnability of
individual {\em concept classes}.  in a generic sense.To this end, we
first define stability for online learning algorithms. Our definition
is essentially ``leave last one out'' stability, also considered by
\cite{RossB11}. We also define a uniform version of this stability
measure.  However, stability alone cannot guarantee bounded
regret. For example, an algorithm that always plays one fixed move is
clearly the most stable any algorithm can be. But its regret can
hardly be bounded.  Hence, an additional condition is required that
forces the algorithm to make {\em progress}. To this end, we introduce
a novel measure called {\em forward regret}: the excess loss incurred
with a look-ahead of one time step (i.e., when player makes its
$t^{\text{th}}$ move \emph{after} seeing the adversary's
$t^{\text{th}}$ move). We show fundamental results relating the three
conditions, namely {\bf online stability}, {\bf bounded forward
  regret} and {\bf bounded regret}. First, assuming stability, bounded
regret and bounded forward regret are equivalent. Second, given an
algorithm with bounded regret, we can always obtain a \emph{stable}
algorithm with bounded regret and bounded forward regret. We would
like to stress that these general results \emph{do not} rely on
convexity assumptions and are \emph{not restricted} to a particular
family of learning algorithms.  In contrast, \cite{RossB11} provides
equivalence of stability and regret for only certain families of
algorithms and concept classes.

We illustrate the usefulness of our general framework by considering
several popular online learning algorithms like Follow-The-Leader
(FTL) \cite{HazAgaKal07,CesLug06}, Follow-The-Regularized-Leader
(FTRL) \cite{Rak09,AbeRak09}, Implicit Online Learning (IOL)
\cite{KulisB10}, Regularized Dual Averaging (RDA) \cite{Xiao10} and
Composite Objective Mirror Descent (COMiD) \cite{DuchiSST10}. We
obtain regret bounds for all of them using the fundamental connections
between forward regret and stability thereby demonstrating that our
framework is not restricted to a particular class of algorithms. Our
regret analysis is arguably simpler than existing ones and, in some
cases such as IOL, provides tighter guarantees as well.


Finally, we consider ``approximate'' versions of RDA, IOL, and FTRL
algorithms where the optimization problem at each step is solved only
up to a small but non-zero additive error.  It is important to
consider such an analysis because, in practice, the optimization
problems arises at each step will not be solved to infinite
precision. For each of these three algorithms, we use our general
stability based recipe to provide regret bounds for their approximate
versions.

We introduce our setup in Section~\ref{sec:prelim}. We introduce the
online learning framework in Section~\ref{sec:setup} and review
existing work and contrast it to our work in
section~\ref{sec:related}. We introduce our three online learning
conditions and show their connections in Section~\ref{sec:method}. We
provide several illustrations of the usefulness of our conditions in
analyzing existing online algorithms in Section~\ref{sec:examples} and
finally conclude with Section~\ref{sec:conclusion}. 




\section{Bregman Divergences and Strong Convexity}
\label{sec:prelim}

Here we recall the definition of a Bregman divergence
\cite{Bre67,CenZen98} which finds use in online learning
algorithms. We also relate it to the notion of strong convexity, a key
property behind many regret bounds for online learning.
\begin{definition}
\vspace*{-5pt}
\as
\bs
  \label{def:breg_div}
  Let $R:\Ccal \to \RR$ be a strictly convex function on a
  convex set $\Ccal \subseteq \RR^d$. Also, let $R$ be differentiable
  on the relative interior of $\Ccal$, $ri(\Ccal)$, assumed to be
  nonempty. The Bregman divergence $\Dcal_{R}:\Ccal \times ri(\Ccal)
  \to \RR^+$ generated by the function $R$ is given by
  \begin{align*}\as \bs
    \Dcal_{R}(\xb,\yb) = R(\xb) - R(\yb)- \grad
      R(\yb)^{\top}(\xb - \yb)
  \end{align*}
  where $\grad R(\yb)$ is the gradient of the function
  $R$ at $\yb$.
\end{definition}

\begin{definition}\as\bs
  \label{def:strong-convex}
  A convex function $f:\RR^{d} \to \RR$ is strongly convex with
  respect to a norm $\nbr{\cdot}$ if there exists a constant $\alpha >
  0$ such that $$\as\bs D_f(\uvec, \vvec)\geq \frac{\alpha}{2}
  \|\uvec-\vvec\|^{2} \qquad \forall \ub, \vb \in \RR^d.$$
\end{definition}
$\alpha$ is called the modulus of strong convexity and $f$ is also
referred to as $\alpha$-strongly convex.

Now, we present a useful lemma characterizing optima of a strongly
convex function.
\begin{lemma}
  \label{lem:sc_property}\as\bs
  Let $f:\RR^{d} \to \RR$ be an $\alpha$-strongly convex function and
  let $\Ccal\subseteq \RR^d$ be a convex set. Let $\wvec^*\in \Ccal$
  be a minimizer of $f$ over $\Ccal$, i.e., $\wvec^*=\argmin_{\wvec\in
    \Ccal}f(\wvec)$. Then, for any $\uvec \in \Ccal$,
$$\as\bs f(\uvec)\geq f(\wvec^*)+\frac{\alpha}{2}\|\uvec-\wvec^*\|^2.$$
In particular, the minimizer is unique.
\end{lemma}

Lower bold case letters (\eg, $\wb$, $\mub$) denote vectors, $w_{i}$
denotes the $i$-th component of $\wb$.
The Euclidean dot product between $\avec$ and $\bvec$ is denoted by
$\avec^{\top}\bvec$ or $\inner{\ab}{\bb}$.  A general norm is denoted
by $\|\cdot\|$ and $\|\cdot\|_*$ refers to its dual norm. For most of
this paper, we work with arbitrary norms and we use $\|\cdot\|_p$ to
refer to a specific $\ell_p$ norm. Unless specified otherwise, $\wvec
\in \mathbb{R}^d$, $\Ccal \subset \mathbb{R}^d$ is a compact convex
set, and $\ell_t:\mathbb{R}^d\rightarrow \mathbb{R}$ is any loss
function. A function $f:\Ccal \to \RR$ is $L$-Lipschitz continuous
w.r.t. a norm $\nbr{\cdot}$ if $|f(\xb)-f(\yb)|\leq L\|\xb-\yb\|,
\forall \xb, \yb\in \Ccal$.

\section{Setup}
\label{sec:setup}

We now describe the online learning setup that we use in this
paper. Let $\Ccal\subset \RR^d$ be a fixed set and $\Lcal$ be a class
of real-valued functions over $\Ccal$. Now, consider a repeated game
of $T$ rounds played between a player/learner and an adversary.  At
every step $t$,
\begin{compactitem}
\item The player plays a point $\wb_t$ from a set $\Ccal$.
\item The adversary responds with a function $\ell_t \in \Lcal$.
\item The player suffers loss $\ell_t(\wb_t)$.
\end{compactitem}
The quantity of interest in online learning is the {\bf regret} which
measures how good the player performs compared to the best fixed move
in hindsight (i.e. knowing all the moves of the adversary in
advance). Regret is defined below in \eqref{eq:bdd_regret}.
The goal in online learning is to minimize the regret regardless of
the function sequence played by the adversary.  Online Convex
Programming (OCP) \cite{Zin03} (respectively Online Linear Programming
(OLP)) is a special case of the online learning game above where the
set $\Ccal$ is a compact convex set and $\Lcal$ is a class of convex
(respectively linear) functions defined on $\Ccal$.


\section{Related Work} 
\label{sec:related}

For a general introduction to online learning and
descriptions of standard algorithms, see~\cite{CesLug06}. In the iid
setting, stability is investigated from various points of view in
\cite{BouEli00,KutinN02,RakhlinMP05,ShaShaSreSri10}.  There are only a
few papers dealing with stability in the online setting.  Recently,
\cite{RossB11} defined what we call Last Leave-One-Out (LLOO)
stability and showed that for FTRL or MD type methods, stable online
learning algorithms have bounded regret. In contrast, we distill out
the ``progress'' in terms of forward regret condition and show a much
more general connection between stability, regret and forward
regret. Unlike \cite{RossB11}, our method is extremely generic and
does not need to assume any specific algorithmic form or even any
specific function class (like convex functions). We also prove that
most existing families of online learning algorithms are in fact {\em
  stable} in our sense and using our connections provide simple regret
bound analysis for them.  Another related work \cite{PogVoiRos11}
considers an online algorithm, namely stochastic gradient descent
(SGD) algorithm, in the iid setting where each function $\ell_t$ is
samples points in an iid fashion from some distribution. In this
setting, \cite{PogVoiRos11} defines a new notion of online stability
which is motivated by uniform stability \cite{BouEli00}. The paper
shows that SGD satisfies the new notion of stability and provides
consistency guarantees as well. In contrast, our fundamental results
connecting stability with regret hold for any algorithm and for any
set of adversary moves $\cbr{\ell_t}$, not just those sampled iid from
a distribution.

A general class of online learning algorithms are referred to as
Follow-The-Leader (FTL) \cite{CesLug06} algorithms. At step $t+1$,
this algorithm chooses the element of $\Ccal$ which minimizes the
sum of the functions played by the adversary up to that
point: \begin{align} \label{eq:FTL} \wb_{t+1} = \argmin_{\wb \in
\Ccal} \sum_{i=1}^{t}f_i(\wb) \ .  \end{align} It can be shown that
surprisingly simple algorithm achieves $O(\log T)$ regret when the
adversary is restricted to playing strongly convex functions
\cite{HazAgaKal07}.

A generalization of FTL is by adding a regularizer which results in
the Follow-The-Regularized-Leader (FTRL) algorithm \cite{Rak09,
  AbeRak09}. In this case the update is given by
\begin{align}
  \label{eq:FTRL}
  \wb_{t+1} = \argmin_{\wb \in \Ccal} \sum_{i=1}^{t}\eta f_i(\wb) +
  R(\wb)
\end{align}
Typically, $R$ is a strongly convex regularizer with respect to the
appropriate norm and $\eta$ is a tradeoff parameter. Another way of
describing FTRL algorithms is using Bregman divergences
\cite{Rak09}. In particular, by defining $\phi_0(\wb) = R(\wb)$ and
$\phi_t(\wb) = \phi_{t-1}(\wb) + \eta f_t(\wb)$, we can write FTRL
update in an equivalent form:
\begin{align*}
  \wb_{t+1} = \argmin_{\wb \in \Ccal} \eta f_{t}(\wb) + \Dcal_{\phi_{t-1}}(\wbtil)
\end{align*}
where $\wbtil$ is the corresponding unconstrained minimizer.

Another class of algorithms is the proximal type algorithms also
called Mirror Descent(MD) methods \cite{NemYud83}, that typically tries to
find an iterate close to the previous iterate but also minimizes the
current loss function and obtains same rates of regret as FTRL. Similar
to FTRL, such algorithms also achieves $O(\sqrt{T})$ regret for
general convex functions and $O(\ln T)$ regret for strongly convex
functions.
It is interesting to note that Zinkevich's algorithm \cite{Zin03} is
just a special case of mirror descent with the Euclidean norm and
$R(\wb) = \frac{1}{2}\nbr{\wb}_2^2$ and is similar to a stochastic
gradient descent update \cite{BotBou07}.

While mirror descent and FTRL look fundamentally different
algorithms and were considered to be two different ends of the
spectrum for online learning algorithms \cite{ShwartzK08}, a recent
paper \cite{McM11} shows equivalence between different mirror descent
algorithms and corresponding FTRL counterparts. In particular they
show that the FOBOS mirror descent algorithm \cite{DucSin09} is
conceptually similar to Regularized Dual Averaging (RDA) \cite{Xiao10}
with minor differences emanating out of usage of proximal strongly
convex regularizer and handling of arbitrary nonsmooth regularization
like the $\ell_1$ norm. These difference result in different sparsity
properties of the two algorithms.

\section{Three conditions for online learning}
\label{sec:method}
In this section, we formally define our stability notion as well as
introduce our bounded forward regret condition. We show that given
stability, bounded regret and bounded forward regret are
equivalent. Moreover, any algorithm with bounded regret can be
converted into a {\em stable} algorithm with bounded regret and
forward regret.  Finally, we consider several existing OCP algorithms
and illustrate that our forward regret and stability conditions can be
used to provide a simple recipe for proving regret. For each of the
algorithms, our novel analysis simplifies existing analysis
significantly and in some cases also tightens the analysis.

We first define the following three quantities for any online
learning algorithm:
\begin{compactitem}
\item {\bf Online Stability}: 
  Intuitively, an online algorithm $\Acal$ is defined to be stable if
  the {\em consecutive} iterates generated by $\Acal$ are not too far
  away from each other. Formally, if $\wb_t$ is the point selected by
  $\Acal$ at the $t$-th step, then the (cumulative) online
  stability of $\Acal$ is given by 
  \begin{align}\as\bs
    \label{eq:stability}
    \Scal_\Acal(T)=\sum_{t=1}^T \nbr{\wb_t-\wb_{t+1}}.
  \end{align}
  Now, if $\Scal_\Acal(T) = o(T)$, then we say that $\Acal$ is online
  stable. 
  of stability is closely related to \cite{RossB11} (See Definition
  17). Next, we define a stronger definition of stability, which we
  call {\bf Uniform Stability}:
  \begin{align}\as\bs
    \label{eq:stability1}
    \Ucal\Scal_\Acal(t)= \nbr{\wb_t-\wb_{t+1}}.
  \end{align}
  If $\Ucal\Scal_\Acal(T) = o(1)$, then $\Acal$ is defined to be
  uniformly stable. Clearly, if $\Acal$ is uniformly stable then it is
  (cumulatively) stable as well. In section~\ref{sec:examples},
  we show that most of the existing online learning methods are
  actually uniformly stable. Interestingly, for COMiD (see
  section~\ref{sec:comid_example}), while proving cumulative stability
  is relatively straightforward, one can show that uniform stability
  need not hold in general.
\item {\bf Forward Regret}: Forward regret is the hypothetical regret
  incurred by $\Acal$ if it had access to the next move that the
  adversary was going to make.
  Note that forward regret cannot actually be attained by an algorithm
  since it depends on seeing one step into the future. Formally, 
  \begin{align}\as\bs
    \label{eq:fwd_regret}
    \Fcal\Rcal_\Acal(T)=\sum_{t=1}^T \sbr{\ell_t(\wb_{t+1}) -
      \ell_t(\wb^*)},
  \end{align}
  where $\wb^* = \argmin_{\wb \in \Ccal} \sum_{t=1}^T \ell_t(\wb)$.
  We define $\Acal$ to have bounded (or vanishing) forward regret if
  $\Fcal\Rcal_\Acal(T)$ $= o(T)$. Note that if the online algorithms
  are randomized, we can replace the three quantities with their
  expected counterparts and all the bounds in the paper still hold.
\item {\bf Regret}: Regret is a standard notion in online learning
  that measures how good the steps of the algorithm $\Acal$ are
  compared to the best fixed point in hindsight: 
  \begin{align}\as\bs
    \label{eq:bdd_regret}
    \Rcal_\Acal(T)=\sum_{t=1}^T \sbr{\ell_t(\wb_{t}) -
      \ell_t(\wb^*)}.
    \as\bs
  \end{align}
  Here again, if $\Rcal_\Acal(T) = o(T)$, then $\Acal$ is
  said to have bounded (or vanishing) regret.
\end{compactitem}
These three concepts, besides being important in their own right, are
also intimately related. In particular, in the next section we show
that given any two of these conditions, the third condition holds.

\subsection{Connections between the three conditions}
\label{sec:equivalence}
In this section, we show that the three conditions (i.e., bounded
stability, bounded forward regret and bounded regret) defined in the
previous section are closely related in the sense that given any two
of the conditions, the third condition follows directly. For our
claim, we first show that {\em assuming stability},
\begin{align*}\as\bs
  \text{bounded forward regret} \iff \text{bounded regret} \ .
\end{align*}
We then prove that bounded regret can be shown to
exhibit stability, albeit with worse rates of regret. Our claims are
formalized in the following theorems.
\begin{theorem}
  \label{thm:equivalence}
  Assume an online algorithm $\Acal$ satisfies the condition of online
  stability \eqref{eq:stability} where the function played by the
  adversary at each step is $L$-Lipschitz. Then, we have,
  \begin{align}\as\bs
    \label{eq:bound_regret}
    &\Rcal_{\Acal}(T)\leq L\cdot \Scal_{\Acal}(T)+\Fcal\Rcal_{\Acal}(T),\\
    \notag &\Fcal\Rcal_{\Acal}(T)\leq L\cdot
    \Scal_{\Acal}(T)+\Rcal_{\Acal}(T).
  \end{align}
  Therefore, assuming online stability of $\Acal$, bounded forward
  regret and bounded regret are equivalent conditions.
\end{theorem}
\begin{proof}
  We first assume that $\Acal$ has online stability and bounded
  forward regret. We have
  \begin{align*}\as\bs
    \sum_{t=1}^T \sbr{\ell_t(\wb_{t}) - \ell_t(\wb^*)} &=
    \sum_{t=1}^T\sbr{\ell_t(\wb_t) - \ell_t(\wb_{t+1})} + \sum_{t=1}^T
    \sbr{\ell_t(\wb_{t+1}) - \ell_t(\wb^*)} \\\as\bs
    &\hspace*{-25pt}\leq \sum_{t=1}^T L\nbr{\wb_t - \wb_{t+1}} +
    \Fcal\Rcal(T) \leq L\cdot \Scal(T) + \Fcal\Rcal(T)= o(T),
  \end{align*}
  where the second last inequality follows by Lipschitz continuity of
  $\ell_t$ and the last equality holds as both
  $\Scal(T),\ \Fcal\Rcal(T) = o(T)$.  Hence, $\Acal$ has bounded
  regret. The proof in the reverse direction follows
  identically. \qedhere
\end{proof}
To complete the 
picture regarding the connections between the three conditions, we now
prove the following theorem.
\begin{theorem}
  \label{thm:eq_regret_stability}
  Let $\Ccal$ be a fixed set of bounded diameter $D$ from which a
  learner $\Acal$ selects a point at each step of online learning. Let
  $\Fcal$ be the class of $L$-Lipschitz functions from which the
  adversary plays a function at each step.  Also, let $\Acal$ have
  bounded regret. Then, there exists a stable algorithm with
  bounded regret and forward regret.
\end{theorem}
\begin{proof}
Intuitively, our proof proceeds by constructing an alternative stable
algorithm that averages a batch of loss functions and feeds it into
the ``unstable'' but bounded regret algorithm $\Acal$. We then show
bounded regret and forward regret of this new algorithm. Note that our
proof strategy is inspired by the proof of Lemma 20 in
\cite{ShaShaSreSri10} that shows stability to be a necessary condition
for learnability in batch setting.
 
Formally, given the algorithm $\Acal$, we construct a new algorithm
$\Acal '$ in the following way. We divide the set of points into
batches of $B$ and $\Acal'$ repeats the same point in an entire
batch. At the end of the batch, it feeds the average of the functions
in the batch to $\Acal$ to get its next move.  It now sticks to this
new point for the next $B$ time steps before repeating the process all
over. In picture,
\begin{align*}\as\bs
  \underbrace{\Acal' \text{ sees}:}_{\Acal \text{ sees}:} \qquad
  \underbrace{\ell_1,\ell_2,\hdots, \ell_B}_{B \cdot g_1},
  \underbrace{\ell_{B+1},\hdots, \ell_{2B}}_{B \cdot g_2},\underbrace{\hdots}_{\hdots}
\end{align*}
Note that the function $g_t$, being an average of Lipschitz
functions, is itself Lipschitz.  Denote the elements generated by
$\Acal'$ as $\wb_1',\ldots,\wb_T'$ and those by $\Acal$ as $\wb_1,
\ldots, \wb_{\lfloor T/B\rfloor}$. Note that there are only $\lfloor
T/B \rfloor$ distinct elements $\wb_1,\ldots,\wb_{\lfloor T/B
  \rfloor}$ in this sequence: viz.  the elements generated by
$\Acal$ in response to $g_1,\ldots,g_{\lfloor T/B \rfloor}$. The
stability analysis of $\Acal '$ now proceeds as follows
\vspace*{-5pt}
\begin{align*}\as\bs
  \sum_{t=1}^T\nbr{\wb_t ' - \wb_{t+1} '} = \sum_{t=1}^{\lfloor T/B
    \rfloor}\nbr{\wb_{(t-1)B+1}' - \wb_{tB+1}'} =
  \sum_{t=1}^{\lfloor T/B \rfloor} \nbr{ \wb_t - \wb_{t+1} } \leq
  \frac{T}{B}D = o(T),\as\bs
\end{align*}
for the choice $B = O(\sqrt{T})$ in particular. This proves that
$\Acal '$ is stable. 

In order to show that $\Acal '$ has bounded regret, we consider
\begin{align*}\as\bs
  \hspace*{-13pt}\sum_{t=1}^T \rbr{\ell_t(\wb_t ') - \ell_t(\wb^*)}
  &\leq \sum_{t=1}^{B\lfloor T/B \rfloor} \rbr{\ell_t(\wb_t') -
    \ell_t(\wb^*)} + L\cdot D\cdot B = \sum_{i=1}^{\lfloor T/B \rfloor}
 \sum_{t=(i-1)B+1}^{iB} \rbr{\ell_t(\wb_t') - \ell_t(\wb^*)} +LDB \\ 
 &= B \sum_{i=1}^{\lfloor T/B \rfloor} \rbr{g_i(\wb_i) - g_i(\wb^*)} +
 L\cdot D\cdot B \leq B \cdot\Rcal_\Acal(\lfloor T/B \rfloor) + L\cdot
 D\cdot B,
\end{align*}
where $\Rcal_\Acal(T) = o(T)$ as $\Acal$ has bounded regret. The
last term in the first inequality is an upper bound on the regret
due to the last batch of functions (maximally $B$ in
number). Selecting $B = \sqrt{T}$, we get $\Rcal_{\Acal}(\lfloor
T/B \rfloor) = o(\sqrt{T})$ and hence the above bound is $o(T)$,
i.e, $\Acal'$ has bounded regret.
\end{proof}


  Thus we show that given any algorithm with bounded regret, we can
  convert it into another online stable algorithm with bounded regret
  which also implies bounded forward regret using Theorem
  \ref{thm:equivalence}.



\section{Unified analysis of online algorithms}
\label{sec:examples}
In this section we present examples where existing online learning
algorithms can be analyzed through our stability and forward regret
conditions and hence lead to regret bounds directly (see
Theorem~\ref{thm:equivalence}). These examples illustrate that the
stability and forward regret conditions are critical to regret
analysis and in fact provide a fairly straightforward recipe for
regret analysis of online learning algorithms. Note that, unlike the
general results of section \ref{sec:method}, here we will make
convexity assumptions on $\Ccal$ and $\ell_t$. One of the major
contributions of this paper is that our analysis significantly
simplifies as well as tightens up analysis for existing methods like
IOL \cite{KulisB10}.  

Before delving into the technical detials, we provide a brief generic
sketch of the regret analysis of all the algorithms. 

  For each of the regret analyses, initially we bound the stability
  $\sum_t\nbr{\wb_t -\wb_{t+1}}$ in terms of the learning rate $\eta$
  and the Lipschitz coefficient of $\ell_t$, $L$. The bounds on
  stability are generally obtained by exploiting the optimality of
  $\wb_{t+1}$ at iteration $(t+1)$, the lipschitz continuity of
  $\ell_t$ and the strong convexity of the regularizer $R$ (for the
  algorithms involving regularization). For the case of IOL,
  $\nbr{\wb_t -\wb_{t+1}} \leq 2L\eta_t$, which makes the stability
  bounded by $2L\sum_t\eta_t$.
  
  For FTL, forward regret is non positive by definition of the FTL
  updates. For all the other algorithms, the bounds on the forward
  regret follow by again using the optimality of $\wb_{t+1}$ at
  $(t+1)^{th}$ iteration and comparing the corresponding objective at
  the final minimizer $\wb^*$. This generally results in a telescoping
  sum, upper bounding the forward regret in terms of the regularizer
  $R$ (or the bregman divergence $\Dcal_R$) evaluated at the extreme
  iterates $\wb_T$ and $\wb_1$ with all the other terms canceling out
  by appropriately choosing $\eta_t$. In particular, for the case of
  IOL, the forward regret is bounded by
  $\frac{1}{\eta_1}\Dcal_R(\wvec^*,\wvec_1) + \sum_{t=2}^T
  \left(\frac{1}{\eta_{t}}-\frac{1}{\eta_{t-1}}-\alpha\right)\Dcal_R(\wvec^*,\wvec_{t})$.
 
Finally bounds on the regret are obtained by using equation
\eqref{eq:bound_regret} while the optimum dependence on $T$ are
obtained by trading off the step size $\eta_t$ in the corresponding
inequality. Summation over appropriate $\eta_t$ gives us $O(\log T)$
rates of regret for strongly convex $\ell_t$ and $O(\sqrt{T})$ rates
of regret for general convex lipschitz $\ell_t$ as is common in the
literature.

\section{Examples}
\label{sec:app_examples}
\subsection{Follow The Leader (FTL)}
\label{sec:FTL_example}
Follow the leader(FTL) is a popular method for OCP when the provided
functions are strongly convex. At the $t$-th step FTL chooses
$\wvec_{t+1}\in \mathcal{C}$ to be the element that minimizes the
total loss up to that step, \ie,
\begin{equation}\as\bs
  \label{eq:ftl}
  \text{FTL}:\qquad  \wvec_{t+1}=\argmin_{\wvec\in \Ccal}\sum_{\tau=1}^t \ell_\tau(\wvec). 
\end{equation}
The FTL method was analyzed in \cite{CesLug06} and \cite{ShwartzK08}
for the case when each loss function $\ell$ is at least
$\alpha$-strongly convex. Here, using our forward regret and stability
conditions, we provide a significantly simpler analysis with similar
regret bounds. It should be noted that our analysis is a
generalization of the analysis in \cite[Section 3.2]{CesLug06} from
strongly convex functions w.r.t. $L_2$ norm to strongly convex
functions w.r.t. arbitrary norm. 
\begin{theorem}\as\bs
  Let each loss function $\ell_t$ be $\alpha$-strongly convex and
  $L$-Lipschitz continuous. Then, the regret incurred by FTL algorithm
  (see \eqref{eq:ftl}) is bounded by:
$$\as\bs\Rcal_{\text{FTL}}(T)\leq \frac{2L^2}{\alpha}(1+\ln T).$$
\end{theorem}
\begin{proof}
  Our proof follows the simple recipe of computing stability as well as
  forward regret bound.

  {\bf Stability}: Using strong convexity, Lemma \ref{lem:sc_property}
  and the fact $\wvec_{t+1}$ is the optimum of \eqref{eq:ftl},
  \begin{equation}\as\bs
    \label{eq:ftl_optima}
    \sum_{\tau=1}^t \ell_\tau(\wvec_t) \geq \sum_{\tau=1}^t \ell_\tau(\wvec_{t+1})+\frac{t\alpha}{2}\|\wvec_{t}-\wvec_{t+1}\|^2.
  \end{equation}
  Similarly, using optimality of $\wvec_t$ for the $t-1$-th step:
  \begin{equation}\as\bs
    \label{eq:ftl_optima1}
    \hspace*{-10pt}\sum_{\tau=1}^{t-1} \ell_\tau(\wvec_{t+1}) \geq \sum_{\tau=1}^{t-1} \ell_\tau(\wvec_{t})+\frac{(t-1)\alpha}{2}\|\wvec_{t}-\wvec_{t+1}\|^2.
  \end{equation}
  Adding \eqref{eq:ftl_optima} and \eqref{eq:ftl_optima1}, and by
  using Lipschitz continuity of $\ell_t$ we get:
  \begin{align}\as\bs
    \ell_t(\wvec_{t}) - \ell_t(\wvec_{t+1})&\geq (t-1/2)\alpha\|\wvec_{t}-\wvec_{t+1}\|^2,\nonumber \\
    \implies \qquad \frac{L}{(t-1/2)\alpha}&\geq \|\wvec_{t}-\wvec_{t+1}\|. \as\bs
    \label{eq:ftl_optima2}
  \end{align}
  Using \eqref{eq:ftl_optima2}, we get:
  \begin{align}\as\bs\label{eq:ftl_optima3}
   \hspace*{-10pt} \sum_{t=1}^T \|\wvec_{t}-\wvec_{t+1}\| & \leq \sum_{t=1}^T\frac{2L}{(2t-1)\alpha} \leq  \frac{2L}{\alpha}(1 + \ln T). \as\bs
  \end{align}
  Hence,
\begin{equation}\as\bs
  \label{eq:ftl_stab}
  \Scal_{\text{FTL}}(T)\leq \frac{2L}{\alpha}(1+\ln T). 
\end{equation}
{\bf Forward Regret}: Using optimality of $\wvec_{t+1}$ for $t$-th step:
\begin{equation}\as\bs
  \label{eq:ftl_fw1}
  \sum_{t=1}^T\ell_t(\wvec^*) \geq \sum_{t=1}^{T}\ell_t(\wvec_{T+1}).
\end{equation}
Next using \eqref{eq:ftl_optima1} for $t=T$ and \eqref{eq:ftl_fw1},
\begin{equation}\as\bs
  \label{eq:ftl_fw2}
  \sum_{t=1}^{T}\ell_t(\wvec^*) \geq \ell_T(\wvec_{T+1})+\sum_{\tau=1}^{T-1}\ell_\tau(\wvec_{T}).
\end{equation}
Similarly using \eqref{eq:ftl_optima1} with \eqref{eq:ftl_fw2} for
$t=T-2, \dots, 1$,
\begin{equation}\as\bs
  \label{eq:ftl_fw}
  \sum_{t=1}^{T}\ell_t(\wvec^*) \geq \sum_{t=1}^T \ell_t(\wvec_{t+1}).
\end{equation}
Hence, 
\begin{equation}\as\bs
  \label{eq:ftl_fr}
  \Fcal\Rcal_{\text{FTL}}(T)=0. 
\end{equation}
Hence, using Theorem~\ref{thm:equivalence}, \eqref{eq:ftl_stab}, and \eqref{eq:ftl_fr}, 
\begin{equation}\as\bs
  \label{eq:ftl_r}
  \Rcal_{\text{FTL}}(T)\leq \frac{2L^2}{\alpha}(1+\ln T). 
\end{equation}
\end{proof}
\vspace*{-15pt}

\subsection{Follow The Regularized Leader (FTRL)}
\label{sec:FTRL_example}
While FTL is an intuitive algorithm, unfortunately, for non-strongly
convex functions it need not have bounded regret. However, several
recent results show that by adding strongly convex regularization, FTL
can be used to obtain bounded regret. Specifically,
\begin{equation}\as\bs
  \label{eq:ftrl}
  \text{FTRL}:\qquad  \wvec_{t+1}=\argmin_{\wvec\in \Ccal}\sum_{\tau=1}^t \ell_\tau(\wvec)+\frac{1}{\eta}R(\wvec). 
\end{equation}
where $R$ is (generally) a strongly convex function with respect to
an appropriate norm. Note that the intuition behind adding a
regularization is making the algorithm stable. Our analysis of FTRL
explicitly captures this intuition by showing the existence of
stability condition, while forward regret follows easily from the forward regret analysis of FTL given above.
\begin{theorem}
  Let each loss function $\ell_t$ be $L$-Lipschitz continuous,
  diameter (as measured in $\|\cdot\|$) of set $\Ccal$ be $D$, and let $R$ be a $1$-strongly convex
  regularization function. Then, the regret incurred by Follow The
  Regularized Leader (FTRL) algorithm (see \eqref{eq:ftrl}) is bounded
  by:
$$ \as\Rcal_{\text{FTRL}}(T)\leq 2\, L\sqrt{\|\grad R\|_* D}\sqrt{T} \ ,$$
where $\|\grad R\|_* = \sup_{\wb \in \Ccal} \|\grad R(\wb)\|_*$.
\end{theorem}
\vspace*{-10pt}
\begin{proof}
  As for FTL, we again prove regret by first proving stability and forward regret.\\
  {\bf Stability}: Similar to \eqref{eq:ftl_optima} and
  \eqref{eq:ftl_optima1}, using strong convexity and optimality
  conditions for $t$-th and $t-1$-th step, we get the following
  relations:
\begin{align}\as\bs
    \label{eq:ftrl_optima}
    &\sum_{\tau=1}^t \ell_\tau(\wvec_t) + \frac{1}{\eta} R(\wvec_t)
    \nonumber\\\as\bs &\hspace*{-10pt}\geq \sum_{\tau=1}^t
    \ell_\tau(\wvec_{t+1})+ \frac{1}{\eta}
    R(\wvec_{t+1})+\frac{1}{2\eta}\|\wvec_{t}-\wvec_{t+1}\|^2.\\\as\bs
    \label{eq:ftrl_optima1}
    &\sum_{\tau=1}^{t-1} \ell_\tau(\wvec_{t+1}) + \frac{1}{\eta}
    R(\wvec_{t+1}) \nonumber\\\as\bs &\geq \sum_{\tau=1}^{t-1}
    \ell_\tau(\wvec_{t})+ \frac{1}{\eta}
    R(\wvec_t)+\frac{1}{2\eta}\|\wvec_{t}-\wvec_{t+1}\|^2.
\end{align}
Combining \eqref{eq:ftrl_optima} and \eqref{eq:ftrl_optima1} and
by Lipschitz continuity of $\ell_t$:
\begin{equation}\as\bs
  \label{eq:ftrl_optima2}
  L\eta\geq \|\wvec_{t}-\wvec_{t+1}\|. 
\end{equation}
Hence, 
\begin{equation}\as\bs
  \label{eq:ftrl_stab}
  \Scal_{\text{FTRL}}(T)=\sum_{t=1}^T \|\wvec_{t}-\wvec_{t+1}\| \leq L\eta T. 
\end{equation}
Choosing $\eta = \frac{1}{\sqrt{T}}$ satisfies the online stability condition of FTRL.
 
{\bf Forward Regret}: Assuming $\ell_0(\wvec)=R(\wvec)$ and
$\wvec_1=\argmin_{\wvec \in \Ccal}R(\wvec)$, FTRL is same as FTL with
an additional $0$-th step loss function
$\ell_0(\cdot)=R(\cdot)$. Hence using \eqref{eq:ftl_fw}, we obtain:
\begin{equation}\as\bs
  \label{eq:ftrl_fw}
  \sum_{t=1}^{T}\ell_t(\wvec^*) +\frac{1}{\eta}(R(\wvec^*)-R(\wvec_1)) \geq \sum_{t=1}^T \ell_t(\wvec_{t+1}). 
\end{equation}
Hence, 
\begin{align}\as\bs
  \label{eq:ftrl_fw_final}
  \Fcal\Rcal_{\text{FTRL}}(T)=\frac{1}{\eta}(R(\wvec^*)-R(\wvec_1)) 
  \leq \frac{\grad R(\wb^*)^{\top}(\wb^* - \wb_1)}{\eta} \leq \frac{\|\grad R\|_* D}{\eta}. \as\bs
\end{align}
where the first inequality follows using the convexity of $R$ and the
last one follows using Cauchy Schwartz inequality. Again $\eta =
\frac{1}{\sqrt{T}}$ provides vanishing forward regret for FTRL. Hence,
using Theorem~\ref{thm:equivalence},
\begin{align}\as\bs
  \hspace*{-20pt}\Rcal_{\text{FTRL}}(T)\leq \frac{\|\grad R\|_* D}{\eta} + L^2\eta
  T \leq 2\,L\sqrt{\|\grad R\|_* D}\sqrt{T}.
\end{align}
by appropriately choosing $\eta$ to be $\frac{1}{\sqrt{T}}$.
\end{proof}

\subsection{Regularized Dual Averaging (RDA)}
\label{sec:rda_example}
Regularized Dual Averaging \cite{Xiao10} is a popular online learning
method to handle OCP scenarios where each loss function is regularized
by the same regularization function, i.e., functions at each step are of
the form $\ell_t'(\wvec)=\ell_t(\wvec)+r(\wvec)$, where $r$ is a
regularization function. RDA computes the iterates using following
rule:
\begin{align}\as\bs
  \label{eq:rda}
\hspace*{-20pt}  \text{RDA:}\ \ \wvec_{t+1}=\argmin_{\wvec\in\Ccal}
  \sum_{\tau=1}^t \gvec_\tau^{\top} \wvec+t\cdot r(\wvec) + \beta_t
  h(\wvec),
\end{align}
where $\gvec_t=\grad \ell_t(\wvec_t)$, $h(\wvec)$ is a strongly convex
regularizer that is separately added and $\beta_t$ is the trade-off
parameter. \cite{Xiao10} shows that the above update obtains
$O(\sqrt{T})$ regret for general Lipschitz continuous functions and
$O(\ln T)$ regret when the regularizer $r$ is strongly convex.

Note that RDA is same as FTRL except for linearization of the first
part of loss function $\ell_t$. Hence, same regret analysis as FTRL
should hold. However, analysis by \cite{Xiao10} shows that by using
special structure of $\ell_t'$, regret can be bounded even without
assuming Lipschitz continuity of the regularization function
$r$. Below, we show that using the same recipe of bounding stability
and forward regret leads to significantly simpler analysis of RDA as
well. Unlike the previous cases, this analysis is slightly more tricky
as we cannot assume Lipschitz continuity of $r$ to prove stability.
\begin{theorem}
  \label{thm:RDA_sc}
  Let each loss function $\ell_t$ be $L$-Lipschitz continuous, $r$ be
  a $\alpha$-strongly convex function and wlog $\min_{\wvec\in
    \Ccal}r(\wvec)=0$. Now, using $\beta_t=0$ at each step, regret of
  RDA (see \eqref{eq:rda}) is bounded by $\frac{2L^2}{\alpha}(1+\ln T)$.\vspace*{-10pt}
\end{theorem}
\begin{proof} 
{\bf Stability}: By strong convexity of $r$ and
  optimality of $\wvec_{t+1}$ and $\wvec_t$ for the $t$-th and
  $t-1$-th step respectively,\vspace*{-10pt}
\begin{align*}\as\bs
  \hspace*{-35pt}\frac{1}{t}\sum_{\tau=1}^t \gvec_\tau^{\top} (\wvec_{t}-\wvec_{t+1}) + r(\wvec_{t}) - r(\wvec_{t+1})&\geq \frac{\alpha}{2}\|\wvec_t-\wvec_{t+1}\|^2,\nonumber\\[0pt]
  \hspace*{-35pt}\frac{1}{t-1}\sum_{\tau=1}^{t-1} \gvec_\tau^{\top}
  (\wvec_{t+1}-\wvec_t) + r(\wvec_{t+1}) - r(\wvec_t)&\geq
  \frac{\alpha}{2}\|\wvec_t-\wvec_{t+1}\|^2.\vspace*{-20pt}
\end{align*}
Adding the above two equations, \vspace*{-10pt}
\begin{align}\as\bs
  \hspace*{-15pt} \alpha\|\wvec_t-\wvec_{t+1}\|^2\leq
  \left(\frac{1}{t}\gvec_t-\frac{1}{t(t-1)}\sum_{\tau=1}^{t-1}\gvec_\tau\right)^{\top}(\wvec_t-\wvec_{t+1})
  \leq \frac{2L}{t}\|\wvec_t-\wvec_{t+1}\|,\vspace*{-10pt}
\end{align}
where the second inequality follows from Lipschitz continuity of
$\ell_\tau, 1\leq \tau\leq t$. After simplification and adding the
above expression for all $t=1,\dots,T$,
\begin{equation}\as\bs
  \label{eq:rda_stab}
  \Scal_{\text{RDA}}(T)\leq \frac{2L}{\alpha}(1+\ln T). 
\end{equation}
Note that the above stability analysis is slightly different from that
of FTL as we are able to bound the stability by Lipschitz constant of
$\ell_t$ only, rather than $\ell_t+r$. 

{\bf Forward Regret}: When $\beta_t = 0$, forward regret follows
easily from forward regret of FTL where loss function at each step is
$\gvec_t^{\top} \wvec+ r(\wvec)$. Hence,
\begin{equation}\as\bs
  \label{eq:rda_fr}
  \Fcal\Rcal_{\text{RDA}}(T)\leq 0. 
\end{equation}
Hence, using Theorem~\ref{thm:equivalence}, 
\begin{align*}\as\bs
  \sum_{t=1}^T\left(\gvec_t^{\top}
    (\wvec_t-\wvec^*)+r(\wvec_t)-r(\wvec^*)\right) \leq
  \frac{2L^2}{\alpha}(1+\ln T).
\end{align*}
The result now follows using convexity of $\ell_t$, i.e., $\ell_t(\wvec_t)-\ell_t(\wvec^*)\leq \gvec_t\cdot
(\wvec_t-\wvec^*)$. 
\end{proof}
\vspace*{-10pt}
Next, we bound regret incurred by RDA for general convex, Lipschitz
continuous functions.
\begin{theorem}
  \label{thm:RDA_lip}
  Let each loss function $\ell_t$ be $L$-Lipschitz continuous and wlog
  $\min_{\wvec\in \Ccal}r(\wvec)=0$ and $0 \leq h(\wb) \leq D^2$,
  $\forall \wb \in \Ccal$. Now, using $\beta_t=\sqrt{t}$ at each step,
  regret of RDA (see \eqref{eq:rda}) is bounded by
  $\frac{2L^2}{\alpha}\sqrt{T}$.\vspace*{-10pt}
\end{theorem}
\begin{proof}
{\bf Stability}: Again, by strong convexity of $h$ and
  optimality of $\wvec_{t+1}$ and $\wvec_t$ for the $t$-th and
  $t-1$-th step respectively,
\begin{align*}
  \frac{1}{t}\sum_{\tau=1}^t \gvec_\tau\cdot (\wvec_{t}-\wvec_{t+1}) +
  r(\wvec_{t}) - r(\wvec_{t+1})
  +\frac{\beta_t}{t}(h(\wvec_t)-h(\wvec_{t+1})) \geq
  \frac{\beta_t}{2t}\|\wvec_t-\wvec_{t+1}\|^2,\vspace*{-15pt}
\end{align*}\vspace*{-15pt}
\begin{align*}
  \vspace*{-15pt} \frac{1}{t-1}\sum_{\tau=1}^{t-1} \gvec_\tau\cdot
  (\wvec_{t+1}-\wvec_t) + r(\wvec_{t+1}) - r(\wvec_t)
  +\frac{\beta_{t-1}}{t-1}(h(\wvec_{t+1})-h(\wvec_{t})) \geq
  \frac{\beta_{t-1}}{2(t-1)}\|\wvec_t-\wvec_{t+1}\|^2.
\end{align*}
Adding the above two equations, using Lipschitz continuity of $\ell_t$
and upper bound on $h$,
\begin{align}
  \label{eq:rda_s1}
  \left(\frac{1}{2\sqrt{t}}+
  \frac{1}{2\sqrt{t-1}}\right)\|\wvec_t-\wvec_{t+1}\|^2-\frac{2L}{t}\|\wvec_t-\wvec_{t+1}\|
  -\left(\frac{1}{\sqrt{t-1}}- \frac{1}{\sqrt{t}}\right)D^2\leq 0
\end{align}
Solving for $\|\wvec_t-\wvec_{t+1}\|$, we get,
\begin{equation}
  \label{eq:rda_s2}
  \|\wvec_t-\wvec_{t+1}\|\leq \frac{2L+D}{\sqrt{t-1}}. 
\end{equation}
Hence, 
\begin{equation}
  \label{eq:rda_stab_g}
\Scal_{\text{RDA}}(T)\leq (2L+D)\sqrt{T}. 
\end{equation}
{\bf Forward Regret}: Using optimality of $\wvec_{T+1}$, 
\begin{align}
  \label{eq:rda_fw1}
  \sum_{t=1}^T\gvec_t^{\top}\wvec^*+T r(\wvec^*)+ \sqrt{T} h(\wvec^*)
  \geq \sum_{t=1}^T\gvec_t\cdot \wvec_{T+1} +T r(\wvec_{T+1})+
  \sqrt{T} h(\wvec_{T+1}).
\end{align}
Now, using optimality of $\wvec_{T}$, 
\begin{align}
  \label{eq:rda_fw2}
  \sum_{t=1}^{T-1}\gvec_t^{\top} \wvec_{T+1}+(T-1) r(\wvec_{T+1})+
  \sqrt{T-1} h(\wvec_{T+1}) \geq \sum_{t=1}^{T-1}\gvec_t\cdot
  \wvec_{T} +(T-1) r(\wvec_{T})+ \sqrt{T-1} h(\wvec_{T}).
\end{align}
Adding the above two equations,
\begin{align}
  \label{eq:rda_fw3}
  \sum_{t=1}^T\gvec_t^{\top} \wvec^*+T r(\wvec^*)+ \sqrt{T} h(\wvec^*)
  \geq \gvec_T^{\top} \wvec_{T+1} + r(\wvec_{T+1})+
  \sum_{t=1}^{T-1}\gvec_t^{\top} \wvec_{T} +(T-1) r(\wvec_{T})+
  \sqrt{T-1} h(\wvec_{T}).
\end{align}
Similarly, combining optimality of $\wvec_t, t=T, \dots, 1$ in
\eqref{eq:rda_fw2} recursively with \eqref{eq:rda_fw3},
\begin{align}
  \label{eq:rda_fw4}
  \sum_{t=1}^T\gvec_t^{\top} \wvec^*+T r(\wvec^*)+ \sqrt{T} h(\wvec^*) 
  \geq \sum_{t=1}^T\left(\gvec_t^{\top} \wvec_{t+1} +
    r(\wvec_{t+1})\right).
\end{align}
Hence, using $\min_{\wvec\in \Ccal}r(\wvec)=0$ and $h(\wvec^*)\leq D^2$, 
\begin{equation}
  \label{eq:rda_fr_g}
  \Fcal\Rcal_{\text{RDA}}(T)\leq \sqrt{T}D^2. 
\end{equation}
Hence, using Theorem~\ref{thm:equivalence} and convexity of each $\ell_t$, 
\begin{equation}
  \label{eq:rda_r_g}
  \Rcal_{\text{RDA}}(T)\leq (D^2+L(2L+D))\sqrt{T}. 
\end{equation}
\end{proof}
\subsection{Composite Objective Mirror Descent (COMiD)}
\label{sec:comid_example}
Similar to RDA, COMiD \cite{DuchiSST10} is also designed to handle
regularized loss functions of the form $\ell_t+r$. Just as RDA is an
extension of FTRL to handle composite regularized loss functions,
similarly, COMiD is an extension of IOL. Formally,
\begin{align*}\as\bs
  \text{COMiD}: \wvec_{t+1}=\argmin_{\wvec \in
    \Ccal}\eta(\gb_t^{\top}\wvec + r(\wvec))+\Dcal_R(\wvec,
  \wvec_{t}),
\end{align*}
where $\gvec_t=\grad \ell_t(\wvec_t)$, $\Dcal_R(\cdot,\cdot)$ is the Bregman
divergence with $R$ being the generating function. Now, similar to
RDA, regret analysis of COMiD follows directly from regret analysis of
IOL. However, \cite{DuchiSST10} presents an improved analysis, that can handle non-Lipschitz continuous regularization $r$ as well. Here,
we show that using our stability/forward-regret based recipe, we can
also obtain similar regret bounds with significantly simpler analysis.
\begin{theorem}
  \label{thm:comid}
  Let each loss function be of the form $\ell_t+r$, where $\ell_t$ is
  a $L$-Lipschitz continuous function and $r$ is a regularization
  function. Let diameter of set $\Ccal$ be $D$, and let
  $\Dcal_R(\cdot,\cdot)$ be a Bregman divergence with $R$ being the
  convex generating function. Let $\wb_1 = \argmin_{\wb \in \Ccal}
  r(\wb)$. Also, let $R$ be a positive function. Then, the regret
  incurred by the Composite Objective Mirror Descent (\text{COMiD}) algorithm
  is bounded by:
  \begin{align*}\as\bs
    \Rcal_{\text{COMiD}}(T)\leq L\sqrt{2R(\wvec^*)}\sqrt{T}.
  \end{align*} 
Furthermore, if each function $\ell_t$ is $\alpha$-strongly convex
w.r.t. $\Dcal_R$, then
$$\as\bs\Rcal_{\text{COMiD}}(T)\leq \frac{2L^2}{\alpha}(1+\ln T) +\alpha R(\wvec^*).$$
\end{theorem}
\begin{proof} 
{\bf Stability}: By optimality of $\wvec_{t+1}$,
\begin{align}\as\bs
  &\hspace*{-33pt}\eta_t(\gvec_t\cdot\wvec_t+r(\wvec_t)) \geq \Dcal_R(\wvec_{t+1}, \wvec_{t})+\eta_t (\gvec_t\cdot\wvec_{t+1}+r(\wvec_{t+1})),\nonumber\\
  &\hspace*{-33pt}\implies L\|\wvec_t-\wvec_{t+1}\|+ r(\wvec_t)\geq r(\wvec_{t+1})+\frac{1}{2\eta_t}\|\wvec_t-\wvec_{t+1}\|^2.
\label{eq:comid_st1}
\end{align}
Adding the above inequality for $t=1, \dots, T$ and using the
fact that $r(\wb_1) \leq r(\wb_{T})$ (by the definition of $\wb_1$),
\begin{equation}\as\bs
  \label{eq:comid_st2}
  \sum_{t=1}^T \frac{1}{2L\eta_t}\|\wvec_t-\wvec_{t+1}\|^2 \leq \sum_{t=1}^T \|\wvec_t-\wvec_{t+1}\|. 
\end{equation}
Using Cauchy-Schwarz inequality, 
\begin{align}\as\bs
  (\sum_{t=1}^T
  \frac{1}{\sqrt{2L\eta_t}}\sqrt{2L\eta_t}\|\wvec_t-\wvec_{t+1}\|)^2 \leq
  \sum_{t=1}^T \frac{1}{2L\eta_t}\|\wvec_t-\wvec_{t+1}\|^2
  \sum_{t=1}^T2L\eta_t.
\label{eq:comid_st3}
\end{align}
Using \eqref{eq:comid_st2} and \eqref{eq:comid_st1}, 
\begin{equation}\as\bs
  \label{eq:comid_stab}
  \Scal_{\text{COMiD}}(T)=\sum_{t=1}^T \|\wvec_t-\wvec_{t+1}\|\leq 2L \sum_{t=1}^T\eta_t. 
\end{equation}

{\bf Forward Regret}: Forward regret follows directly from the forward
regret of IOL \eqref{eq:iol_r}, i.e,\vspace*{-10pt}
\begin{align}\as\bs
  \label{eq:comid_fr}
  \Fcal\Rcal_{\text{COMiD}}  &= \sum_{t=1}^T \left(\gvec_t^{\top}(\wvec_{t+1}-\wvec^*) +
    r(\wvec_{t+1})-r(\wvec^*)\right) \\
 &\leq \frac{1}{\eta_1}\Dcal_R(\wvec^*,\wvec_1)+ \sum_{t=2}^T
  \left(\frac{1}{\eta_{t}}-\frac{1}{\eta_{t-1}}-\alpha\right)\Dcal_R(\wvec^*,\wvec_{t}).
\end{align}
Both the regret bounds follow using convexity of each $\ell_t$ and
setting step sizes $\eta_t$ as in IOL (see \eqref{eq:iol_r1},
\eqref{eq:iol_r2}).
\end{proof}

\subsection{Mirror Descent (MD)}
\label{sec:md_example}
Mirror descent algorithms are a generalization of Zinkevich's Gradient
Infinitesimal Gradient Ascent (GIGA) algorithms \cite{Zin03} where
regularization can be drawn from any Bregman distance
family. Formally,
\begin{align}\as\bs
  \label{eq:md}
  \text{MD}: \wvec_{t+1}=\argmin_{\wvec \in \Ccal}\eta_t
  g_t^{\top}\wvec + \Dcal_R(\wvec, \wvec_{t}),
\end{align}
where $\Dcal_R$ is the Bregman divergence generated using $R$. Note that MD update is the same as COMiD with $r=0$. Hence, our
stability analysis as well as $O(\sqrt{T})$ regret analysis for
general convex functions follows directly. However, for strongly
convex functions, our approach does not yield appropriate
forward regret directly; primary reason being linearization of the
function. Instead, we can obtain regret bound using standard approach
(see \cite{Zin03}) and then obtain forward regret bound using Theorem~\ref{thm:equivalence}.


\subsection{Implicit Online Learning (IOL)}
\label{sec:iol_example}
Implicit online learning \cite{KulisB10} is similar to typical Mirror
Descent algorithms but without linearizing the loss
function. Specifically at iteration $t+1$,
\begin{equation}\as\bs
  \label{eq:iol}
  \hspace*{-10pt}\text{IOL}:\qquad \wvec_{t+1}=\argmin_{\wvec \in
    \Ccal}(\Dcal_R(\wvec, \wvec_{t})+\eta_t \ell_t(\wvec)),
\end{equation}
where $\Dcal_R(\cdot,\cdot)$ is a Bregman's divergence with $R$ being
the generating function. It was shown in \cite{KulisB10} that using
any strongly convex $R$, the above update leads to $O(\sqrt{T})$
regret for any Lipschitz continuous convex functions $\ell_t$. This
paper also shows that if $R$ is selected to be squared $\ell_2$-norm
and each function $\ell_t$ is strongly-convex and has Lipschitz
continuous gradient, then $O(\ln T)$ regret can also be
achieved. Below, using our recipe of forward regret and stability we
reproduce significantly simpler proofs for both $O(\sqrt{T})$ as well
as $O(\ln T)$ regret. Furthermore, our $O(\ln T)$ proof requires only
strong-convexity and Lipschitz continuity, in contrast to
strong-convexity and Lipschitz continuity of the {\em gradient} in
\cite{KulisB10}. Also, our analysis can handle any strongly convex
$R$, rather than just the squared $\ell_2$-norm regularizer.
\begin{theorem}\as\bs
  \label{theorem:regret_IOL}
  Let each loss function $\ell_t$ be $L$-Lipschitz continuous,
  diameter of set $\Ccal$ be $D$, and let $\Dcal_R$ be a Bregman
  divergence with $R$ being the strongly convex generating
  function. Also, let $R$ be a positive function. Then, the regret
  incurred by the Implicit Online Learning (IOL) algorithm (see
  \eqref{eq:iol}) is bounded by:
$$\as\bs\Rcal_{\text{IOL}}(T)\leq 2\,L\sqrt{2R(\wvec^*)}\sqrt{T}.$$
  Furthermore, if each function $\ell_t$ is $\alpha$-strongly convex
  w.r.t $\Dcal_R$ \ie\ $\Dcal_{\ell_t}(\ub,\vb) \geq \alpha
  \Dcal_R(\ub,\vb)$, $\forall \ub,\vb \in \Ccal$ , then
$$\as\bs\Rcal_{\text{IOL}}(T)\leq \frac{2L^2}{\alpha}(1+\ln T)+\alpha
  R(\wvec^*).$$
\end{theorem}
\begin{proof}
    Here again, we follow the recipe of proving stability and forward regret. \\
    {\bf Stability}: Stability again follows easily by using optimality
    of $\wvec_{t+1}$ and comparing it to $\wvec_t$. Formally,
    \begin{align}\as\bs
      \eta_t\ell_t(\wvec_t) &\geq \Dcal_R(\wvec_{t+1}, \wvec_{t})+\eta_t \ell_t(\wvec_{t+1}),\nonumber\\
      \eta_t\ell_t(\wvec_t)&\geq \frac{1}{2}\|\wvec_{t+1}- \wvec_{t}\|^2+\eta_t \ell_t(\wvec_{t+1}),\nonumber\\
      2L\eta_t&\geq \|\wvec_{t+1}- \wvec_{t}\|,
      \label{eq:iol_st1}
    \end{align}
    where the first inequality follows by the strong convexity of $R$ and
    the last one follows by using Lipschitz continuity and canceling
    $\|\wvec_{t+1}- \wvec_{t}\|$ from both sides.  Hence,
\begin{equation}\as\bs
  \label{eq:iol_stab}
  \Scal_{\text{IOL}}(T)\leq 2L\sum_{t=1}^T \eta_t. 
\end{equation}

{\bf Forward Regret}: Similarly, forward regret follows by using
optimality of $\wvec_{t+1}$ and comparing it to $\wvec^*$. Formally,
\begin{align}\as\bs
  &(\wvec^*-\wvec_{t+1})^{\top}(\eta_t\grad\ell_t(\wvec_{t+1})+\grad R(\wvec_{t+1})-\grad R(\wvec_t))\geq 0,\nonumber\\\as\bs
  &(\wvec^*-\wvec_{t+1})^{\top}(\grad R(\wvec_{t+1})-\grad R(\wvec_t))\geq\nonumber\\\as\bs
  &\qquad\qquad\qquad\qquad\qquad \eta_t\grad\ell_t(\wvec_{t+1})^{\top}(\wvec_{t+1}-\wvec^*),\nonumber\\\as\bs
  &\Dcal_R(\wvec^*,\wvec_t)-\Dcal_R(\wvec^*,\wvec_{t+1})-\Dcal_R(\wvec_{t+1}, \wvec_t)\geq\nonumber\\\as\bs
  &\qquad\qquad\qquad\qquad \eta_t\grad\ell_t(\wvec_{t+1})^{\top}(\wvec_{t+1}-\wvec^*). 
  \label{eq:iol_fw1}
\end{align}
where \eqref{eq:iol_fw1} follows from the previous step using the
three point inequality \cite{Rak09}. Now, if $\ell_t$ is
$\alpha$-strongly convex w.r.t $\Dcal_R(\cdot,\cdot)$, then,
\begin{equation}\as\bs
  \label{eq:iol_fw2}
  \ell_t(\wvec_{t+1})^{\top}(\wvec_{t+1}-\wvec^*)\geq \ell_t(\wvec_{t+1})-\ell_t(\wvec^*)+\alpha \Dcal_R(\wvec^*,\wvec_{t+1}). 
\end{equation}

Note that strong convexity w.r.t. $\Dcal_R$ is a stronger condition
than the usual strong convexity w.r.t $\ell_2$ norm. Also, for the
first part of the theorem, we can assume $\alpha=0$.

Using \eqref{eq:iol_fw1} and \eqref{eq:iol_fw2}, and
adding over all $T$ steps,
\begin{multline}\as\bs
  \Fcal\Rcal_{\text{IOL}}(T)=\sum_{t=1}^T\ell_t(\wvec_{t+1})-\ell_t(\wvec^*) \leq \frac{1}{\eta_1}\Dcal_R(\wvec^*,\wvec_1)
  + \sum_{t=2}^T \left(\frac{1}{\eta_{t}}-\frac{1}{\eta_{t-1}}-\alpha\right)\Dcal_R(\wvec^*,\wvec_{t}). \as\bs
  \label{eq:iol_fr}
\end{multline}
Hence, using Theorem~\ref{thm:equivalence} with \eqref{eq:iol_stab} and \eqref{eq:iol_fr}, 
\begin{multline}\as\bs
  \label{eq:iol_r}
  \Rcal_{\text{IOL}}(T)\leq 2L^2\sum_{t=1}^T\eta_t + \frac{1}{\eta_1}\Dcal_R(\wvec^*,\wvec_1)
  + \sum_{t=2}^T \left(\frac{1}{\eta_{t}}-\frac{1}{\eta_{t-1}}-\alpha\right)\Dcal_R(\wvec^*,\wvec_{t})\as\bs
\end{multline}
Now, let us first consider the case when $\alpha=0$, \ie, when
functions $\ell_t$ need not be strongly convex. In this case,
selecting each $\eta_t=\eta$ and $\wvec_1=\argmin_{\wvec\in \Ccal}
R(\wvec)$, we can use the optimality of $\wb_1$ to claim $\grad
R(\wb_1)^{\top}(\wb^* - \wb_1) \geq 0$. Coupling this with the
non-negativity of $R$, we get $\Dcal_R(\wb^*,\wb_1) \leq R(\wb^*)$. This
gives:
\begin{equation}\as\bs
  \label{eq:iol_r1}
\hspace*{-20pt}  \Rcal_{\text{IOL}}(T)\leq 2\eta L^2T+\frac{1}{\eta}\Dcal_R(\wvec^*,\wvec_1)\leq 2\,L\sqrt{2R(\wvec^*)T}
\end{equation}
by optimizing over the choice of $\eta$.
Next, for the case when $\alpha>0$, selecting $\eta_t=\frac{1}{\alpha
  t}$ and $\wvec_1=\argmin_{\wvec\in \Ccal} R(\wvec)$,
\begin{equation}\as\bs
  \label{eq:iol_r2}
  \Rcal_{\text{IOL}}(T)\leq \frac{2L^2}{\alpha}(1+\ln T)+\alpha R(\wvec^*). 
\end{equation}
Hence proved.
\end{proof}



\section{Analysis of approximate online algorithms} 
\label{sec:approximate}
We analyze approximate versions of online algorithms where the updates
at every step are not the exact minimizer of the corresponding
objective but approximate ones. In particular, the updates minimize
the objective upto an additive error $\delta_t$ at iteration $t$ as
would be commonly obtained by some iterative optimization method. We
show that even with such approximate updates we can obtain sublinear
regret over $T$ steps for Regularized Dual Averaging (RDA)
\cite{Xiao10}, FTRL as well as IOL. 

Although RDA requires solving an optimization problem at every
step, it is successful in maintaining the sparsity of the intermediate
iterates and thus finds use in a host of applications where sparsity
is essential \cite{Xiao10}. 
 However, it is typically impossible to solve an optimization problem
 exactly at every step. Hence, it is interesting to analyze the
 behaviour of RDA under such approximate updates.
\subsection{Approximate RDA}
\label{sec:approx_RDA}
The exact updates of the original RDA algorithm are given by 
\[
\hspace*{-20pt} \text{RDA:}\ \ \wvec_{t+1}^*=\argmin_{\wvec\in\Ccal}
  \sum_{\tau=1}^t \gvec_\tau^{\top} \wvec+t\cdot r(\wvec) + \beta_t
  h(\wvec) \ . \vspace*{-5pt}
\] 
where $\gb_\tau = \grad \ell_\tau(\wb_\tau)$, the gradient of the loss
function at iteration $\tau$, $r$ is a regularization function which
is part of the objective while $h$ is a strongly convex regularizer
added by the algorithm.  Using $\wb_{t+1}$ to denote the approximate
update in this case we have
\begin{align}
  \label{eq:approx_rda_inequality}
  \sum_{\tau=1}^t \gvec_\tau^{\top} \wvec_{t+1}+t\cdot r(\wvec_{t+1})
  + \beta_t h(\wvec_{t+1}) \leq \sum_{\tau=1}^t \gvec_\tau^{\top}
  \wvec_{t+1}^*+t\cdot r(\wvec_{t+1}^*) + \beta_t h(\wvec_{t+1}^*) +
  \delta_t
  \vspace*{-5pt}
\end{align}
The following theorem bounds the regret for approximate RDA. 
\begin{theorem}
  \label{thm:RDA_lip_approx}
  Let each loss function $\ell_t$ be $L$-Lipschitz continuous and
  \Wlog\ $\min_{\wvec\in \Ccal}r(\wvec)=0$ and $0 \leq h(\wb) \leq
  D^2$, $\forall \wb \in \Ccal$. Now, using $\beta_t=\sqrt{t}$ and
  $\delta_t = O(1/\sqrt{t})$ at each step, the regret of approximate RDA
  is bounded by $O(\sqrt{T})$.
\end{theorem}
\begin{proof}
{\bf Stability:} Using \eqref{eq:rda_s2}, we know that 
\begin{align*}
  \|\wvec_t^*-\wvec_{t+1}^*\|\leq \frac{2L+D}{\sqrt{t-1}}.
\end{align*}
Using the triangle inequality we can bound the gap between the successive 
iterates of the approximate algorithm as 
\begin{align*}
  \|\wvec_t-\wvec_{t+1}\|\leq \nbr{\wb_t - \wb_t^*} + \nbr{\wb_t^* -
    \wb_{t+1}^*} + \nbr{\wb_{t+1} - \wb_{t+1}^*}.
\end{align*}
Using $f_t(\wb) = \gvec_\tau^{\top} \wvec +t\cdot r(\wvec) + \beta_t
h(\wvec)$ we note that the function $f_t$ is $\beta_t\sigma_h$ strongly
convex where $\sigma_h$ is the coefficient of strong convexity of
$h$. Using the optimality of $\wb_{t+1}^*$ we know that
\begin{align*}
  f_t(\wb_t+1) \geq f_t(\wb_{t+1}^*) + \frac{\beta_t\sigma_h}{2}\nbr{\wb_{t+1} -
    \wb_{t+1}^*}^2 
\end{align*} 
Coupling this with \eqref{eq:approx_rda_inequality}, we have 
\begin{align*}
  \nbr{\wb_{t+1} - \wb_{t+1}^*} \leq \sqrt{\frac{2\delta_t}{\beta_t\sigma_h}}
\end{align*}
Similarly we have $\nbr{\wb_{t} - \wb_{t}^*} \leq
\sqrt{\frac{2\delta_{t-1}}{\beta_{t-1}\sigma_h}}$. Combining these we get 
\begin{align*}
  \|\wvec_t-\wvec_{t+1}\|\leq \frac{2L+D}{\sqrt{t-1}} +
  \sqrt{\frac{2\delta_t}{\beta_t\sigma_h}} +
  \sqrt{\frac{2\delta_{t-1}}{\beta_{t-1}\sigma_h}}
\end{align*}
This gives a bound on the stability 
\begin{align*}
  \Scal(T) = \sum_t\nbr{\wb_t - \wb_{t+1}} \leq \sum_t
  \frac{2L+D}{\sqrt{t-1}} + \sum_t2\sqrt{\frac{2\delta_{t-1}}{\beta_{t-1}\sigma_h}}
\end{align*}

{\bf Forward Regret:} We have 
\begin{align*}
  \sum_{\tau=1}^T\gvec_\tau^{\top} \wvec^*+T\cdot r(\wvec^*) + \beta_T
  h(\wvec^*) &\geq \sum_{\tau =1}^T \gvec_\tau^{\top} \wvec_{T+1}^*+T\cdot
  r(\wvec_{T+1}^*) + \beta_T h(\wvec_{T+1}^*)\\
  &\geq \sum_{\tau=1}^T \gvec_\tau^{\top} \wvec_{T+1}+T\cdot
  r(\wvec_{T+1}) + \beta_T h(\wvec_{T+1}) -\delta_T
\end{align*}
Writing up this inequality for all values of $t$ we have 
\begin{align*}
  \sum_{\tau=1}^T\gvec_\tau^{\top} \wvec^*+T\cdot r(\wvec^*) + \beta_T
  h(\wvec^*) \geq \sum_{\tau=1}^T\rbr{\gvec_\tau^{\top} \wvec_{t+1} +
    r(\wvec_{t+1})} + \sum_t \rbr{\beta_t - \beta_{t-1}}h(\wb_{t+1}) -
  \sum_t \delta_t
\end{align*}
Appropriate simplification and using the fact that $\beta_t\geq 0,
\forall t$ and $0 \leq h(\wb) \leq D^2, \forall \wb$ we have 
\begin{align*}
  \Fcal\Rcal(T) \leq \beta_T h(\wb^*) + \sum_t \delta_t \leq \sqrt{T}D^2 + \sum_t\delta_t
\end{align*}
using the fact that $\beta_t = \sqrt{t}$ . 
Thus the regret bound is given by
\begin{align*}
  \Rcal(T) \leq \sum_t \frac{L(2L+D)}{\sqrt{t-1}} + 2L \sum_t
  2\sqrt{\frac{2\delta_{t-1}}{\beta_{t-1}\sigma_h}} + \sqrt{T}D^2 + \sum_t \delta_t
\end{align*}
Using $\delta_t = O(1/\sqrt{t})$ we have that the second term on the
RHS is bounded by $O(T^{1/2})$ while all the other terms are bounded
by $O(T^{1/2})$ which gives the following sublinear regret bound of
$R_T \leq O(T^{1/2})$. 
\end{proof}

\label{sec:app_approx}
\subsection {Approximate FTRL}
Recall the original FTRL algorithm 
\begin{align*}
  \wb_{t+1}^* = \argmin_{\wb \in \Xcal} \sum_{\tau=1}^t \eta_{\tau}l_{\tau}(\wb) + R(\wb)
\end{align*}
where $R$ is the (possibly strongly convex) regularizer. Our algorithm
chooses $\wb_{t+1}$ such that 
\begin{align}
  \label{eq:approx_ftrl_update}
  \sum_{\tau=1}^t \eta_{\tau}l_{\tau}(\wb_{t+1}) + R(\wb_{t+1}) \leq
  \sum_{\tau=1}^t \eta_{\tau}l_{\tau}(\wb_{t+1}^*) + R(\wb_{t+1}^*) +
  \delta_{t+1}
\end{align}
For notational convenience we use the following notation. 
\begin{align*}
  S_t(\wb) = \sum_{\tau=1}^t \eta_{\tau}l_{\tau}(\wb) + R(\wb)
\end{align*}
Since $R$ is strongly convex in $\wb$, $S_t$ is also strongly convex and satisfies 
\begin{align*}
  S_t(\wb_{t+1}) \geq S_t(\wb_{t+1}^*) +
  \inner{S_t(\wb_{t+1}^*)}{\wb_{t+1} - \wb_{t+1}^*} +
  \frac{1}{2}\nbr{\wb_{t+1} - \wb_{t+1}^*}^2
\end{align*}
But $S_t(\wb_{t+1}) \leq S_t(\wb_{t+1}^*) + \delta_t$. Thus
\begin{align*}
  \delta_t \geq \frac{1}{2}\nbr{\wb_{t+1} - \wb_{t+1}^*}^2 \quad
  \implies \quad \nbr{\wb_{t+1}-\wb_{t+1}^*} \leq \sqrt{2 \delta_t}
\end{align*}

{\bf Stability:} Using the standard stability bound of FTRL and
assuming $\eta_t = \eta$ for all $t$, we have
\begin{align*}
  \nbr{\wb_t - \wb_{t+1}} &\leq \nbr{\wb_t - \wb_t^*} + \nbr{\wb_t^* -
    \wb_{t+1}^*} + \nbr{\wb_{t+1}^* - \wb_{t+1}} \\
  &\leq \sqrt{2\delta_t} + L \eta + \sqrt{2\delta_{t+1}} 
\end{align*}
Thus
\begin{align}
  \nonumber 
  \sum_{t=1}^T\nbr{\wb_t - \wb_{t+1}} &\leq LT\eta +
  \sum_{t=1}^T\sbr{\sqrt{2\delta_t} + \sqrt{2\delta_{t+1}}}\\ 
  \label{eq:stability_FTRL}
  &\leq LT\eta + \sum_{t=1}^T\sbr{2\sqrt{2\delta_t}}
\end{align}
where the last step follows by assuming that $\delta_t$ is a strictly
decreasing sequence in $t$.

{\bf Forward Regret:} We have
\begin{align*}
\sum_{\tau=1}^t l_{\tau}(\wb_{t+1}^*) + \frac{1}{\eta}R(\wb_{t+1}^*)
\leq \sum_{\tau=1}^t l_{\tau}(\wb^*) + \frac{1}{\eta}R(\wb^*)
\end{align*}
Using \eqref{eq:approx_ftrl_update} and telescoping we get
\begin{align*}
\sum_{t=1}^T l_t(\wb_{t+1}) - l_t(\wb^*) \leq
\frac{1}{\eta}\rbr{R(\wb^*) - R(\wb_1)} + \sum_t\frac{\delta_t}{\eta}
\end{align*}
Using the convexity of $R$ and Cauchy Schwartz inequality we have 
\begin{align*}
  R(\wb^*) - R(\wb_1) \leq \inner{\grad R(\wb^*)}{\wb_1 - \wb^*} \leq
  \nbr{\grad R}\nbr{\wb^* - \wb_1} \leq GD
\end{align*}
Thus $\Fcal\Rcal(T) \leq \frac{GD}{\eta} + \sum_t \frac{\delta_t}{\eta}$.
Using the stability theory we have 
\begin{align*}
  \Rcal(T) \leq L \Scal(T) + \Fcal\Rcal(T) 
\end{align*}
Choosing $\eta = \frac{1}{\sqrt{T}}$ and $\delta_t = \delta = 
\frac{1}{T}$, we get 
\begin{align*}
  R_T &\leq GD\sqrt{T} + L^2 \sqrt{T} +
  \sum_{t=1}^T\sbr{\frac{\delta}{\eta} + 2\sqrt{2\delta}} \\ 
  &\leq GD\sqrt{T} + L^2 \sqrt{T} + \frac{T
    \delta^{3/4}}{\eta^{1/2}} \\
\end{align*}
Using $\delta^{3/4} = O(T^{-3/4})$ we get that $R_T =
O(T^{1/2})$. Note that the last line uses the AM-GM inequality which
is only attained at equality that justifies the values of $\eta$ and
$\delta$.

\subsection{Approximate IOL}
The updates of the original IOL algorithm are given by 
\begin{align}
  \label{eq:iol_update}
  \wb_{t+1}^* = \argmin_{\wb \in \Xcal}\sbr{ \eta_t\ell_t(\wb) +
    \frac{1}{2}\nbr{\wb - \wb_t}^2}
\end{align}
We use $f_t(\wb)$ to denote $\eta_t\ell_t(\wb) + \frac{1}{2}\nbr{\wb -
  \wb_t}^2$ in the sequel.  Similar to the FTRL case, we assume that
  $\wb_{t+1}$ is a $\delta_t$ approximate solution. Thus
\begin{align}
  \label{eq:approx_iol_cond}
  \eta_t \ell_t(\wb_{t+1}) + \frac{1}{2}\nbr{\wb_{t+1} - \wb_t}^2
    \leq \eta_t \ell_t(\wb_{t+1}^*) + \frac{1}{2}\nbr{\wb_{t+1}^* -
      \wb_t}^2 + \delta_t
\end{align}
Since $\wb_{t+1}^*$ is optimal we have 
\begin{align*} 
  \inner{\grad f_t(\wb_{t+1}^*)}{\wb_{t+1} - \wb_{t+1}^*} \geq 0 
\end{align*}
Using the optimality of $\wb_{t+1}^*$ and the strong convexity of
$f_t$ we have
\begin{align*}
  \eta_t \ell_t(\wb_{t+1}^*) + \frac{1}{2}\nbr{\wb_{t+1}^* - \wb_t}^2
  + \frac{1}{2}\nbr{\wb_{t+1}^* - \wb_{t+1}}^2 &\leq \eta_t
  \ell_t(\wb_{t+1}) + \frac{1}{2}\nbr{\wb_{t+1} - \wb_t}^2\\ 
    &\leq \eta_t \ell_t(\wb_{t+1}^*) + \frac{1}{2}\nbr{\wb_{t+1}^* -
      \wb_t}^2 + \delta_t
\end{align*}
Simplifying we get
\begin{align}
  \label{eq:iol_cond1}
  \nbr{\wb_{t+1} - \wb_{t+1}^*} \leq \sqrt{2\delta_t}
\end{align}
{\bf Forward Regret:} Denoting $\wb^*$ as the minimizer after $T$
steps we have using optimality of $\wb_{t+1}^*$ and the strong
convexity of $f_t$,
\begin{align*}
  \eta_t \ell_t(\wb_{t+1}^*) + \frac{1}{2}\nbr{\wb_{t+1}^* - \wb_t}^2
  + \frac{1}{2}\nbr{\wb_{t+1}^* - \wb^*}^2 \leq \eta_t \ell_t(\wb^*) +
  \frac{1}{2}\nbr{\wb^* - \wb_t}^2
\end{align*}
Now 
\begin{align*}
  \nbr{\wb^* - \wb^*_{t+1}}^2 &= \nbr{\wb^* - \wb_{t+1} + \wb_{t+1}
    -\wb_{t+1}^*}^2 \\ 
  &\geq \nbr{\wb^* - \wb_{t+1}}^2 + \nbr{\wb_{t+1} - \wb_{t+1}^*}^2
  -2\nbr{\wb^* - \wb_{t+1}}_2\nbr{\wb_{t+1} - \wb_{t+1}^*}_2
\end{align*}
Using the fact that $\nbr{\wb_{t+1} - \wb^*}_2 \leq D$, the diameter
of the set, we get
\begin{align}
  \label{eq:iol_cond2}
  \nbr{\wb^* - \wb_{t+1}^*}^2 \geq \nbr{\wb^* - \wb_{t+1}}^2 -
  2D\sqrt{\delta_t}
\end{align}
Combining \eqref{eq:iol_cond1} and \eqref{eq:iol_cond2} we get 
\begin{align*}
  \eta_t \ell_t(\wb_{t+1}^*) + \frac{1}{2}\nbr{\wb_{t+1}^* - \wb_t}^2
  + \frac{1}{2}\nbr{\wb^* - \wb_{t+1}}^2 \leq \eta_t \ell_t(\wb^*) +
  \frac{1}{2}\nbr{\wb^* - \wb_t}^2 + D\sqrt{2\delta_t}
\end{align*} 
Using \eqref{eq:approx_iol_cond} we have 
\begin{align*}
  \eta_t \ell_t(\wb_{t+1}) + \frac{1}{2}\nbr{\wb_{t+1} - \wb_t}^2
  + \frac{1}{2}\nbr{\wb^* - \wb_{t+1}}^2 \leq \eta_t \ell_t(\wb^*) +
  \frac{1}{2}\nbr{\wb^* - \wb_t}^2 + \delta_t + D\sqrt{2\delta_t}
\end{align*}
This can be rewritten as 
\begin{align*}
  \eta_t \ell_t(\wb_{t+1}) \leq \eta_t \ell_t(\wb^*) +
  \frac{1}{2}\sbr{\nbr{\wb^* - \wb_t}^2 - \nbr{\wb^* -\wb_{t+1}}^2} -
  \nbr{\wb_{t+1} - \wb_t}^2 + \delta_t + D\sqrt{2\delta_t}
\end{align*}
Adding up the above inequality for $t=1\hdots T$ and assuming $\eta_t
= \eta$ we note that some of the terms on the RHS cancel out by
telescoping. Using the fact that $\nbr{\wb^* - \wb_1} \leq D$ this
gives
\begin{align*}
  \sum_t \ell_t(\wb_{t+1}) \leq \sum_t \ell_t(\wb^*) +
  \frac{D^2}{2\eta} + \frac{\sum_t \delta_t}{\eta} +
  \frac{D \sqrt{2 \delta_t}}{\eta}
\end{align*}
Thus we have forward regret 
\begin{align} 
  \label{eq:fr_iol_approx}
  \Fcal\Rcal(T) \leq \frac{D^2}{2\eta} + \frac{\sum_t \delta_t}{\eta} +
  \frac{D \sqrt{2 \delta_t}}{\eta}
\end{align}

{\bf Stability:} Using the strong convexity of the objective we have
\begin{align*}
\eta_t \ell_t(\wb_{t+1}^*) + \frac{1}{2}\nbr{\wb_{t+1}^* - \wb_t}^2
  + \frac{1}{2}\nbr{\wb_{t+1}* - \wb_t}^2 \leq \eta_t \ell_t(\wb_t)
\end{align*}
Using the fact that $\ell_t$ is $L-$ lipschitz continuous we have 
\begin{align*}
  \nbr{\wb_{t+1}^* - \wb_t} \leq L\eta_t 
\end{align*}
Using \eqref{eq:iol_cond1} and the triangle inequality we get
\begin{align} 
  \label{eq:approx_stability_iol}
  \nbr{\wb_{t+1} - \wb_t} \leq L \eta_t + \sqrt{2 \delta_t} 
\end{align}
Combining stability and forward regret we get
\begin{align*}
  \Rcal(T) &\leq L \Scal(T) + \Fcal\Rcal(T) = \frac{D^2}{2\eta} + \frac{\sum_t
    \delta_t}{\eta} + \frac{D \sqrt{2 \delta_t}}{\eta} + L^2 \eta T +
  L \sum_t \sqrt{2\delta_t}
\end{align*} 
Using the fact that $\delta_t \leq \sqrt{\delta_t}$ 
we have
\begin{align}
  \nonumber
  R_T &\leq \frac{D}{2\eta} + \frac{\sum_t\sqrt{2\delta_t}}{\eta} +
  \frac{D\sqrt{2\delta_t}}{\eta} + 2L^2T\eta +
  L\sum_t\sqrt{2\delta_t}\\
  \label{eq:approx_iol_penultimate}
  &\leq 2\rbr{2L^2 T(D+ \sum_t\sqrt{\delta_t} + D \sqrt{2\delta_t})}^{1/2}
  + L\sum_t \sqrt{2\delta_t}
\end{align}
Setting $\delta_t = 1/t$ we have
$\sum_t\sqrt{\delta_t} = O(T^{1/2})$. Replacing it in
\eqref{eq:approx_iol_penultimate}, we have
\begin{align*} 
  R_T \leq O(2LT^{3/4}) + O(LT^{1/2}) = O(L\sqrt{D}T^{3/4})
\end{align*}
thus giving sublinear regret for the IOL algorithm.

On the other hand, setting $\delta_t = 1/t^2$ gives
$\sum_t \sqrt{\delta_t} = O(\log T)$. Replacing this in
\eqref{eq:approx_iol_penultimate}, we get
\begin{align*} 
  R_T \leq \tilde{O}(2L T^{1/2}) + \tilde{O}(L) =
  \tilde{O}(L\sqrt{D}T^{1/2})
\end{align*}
where $\tilde{O}$ hides logarithmic factors in $T$. 

While we provide rates on $\delta_t$ for getting regret bounds akin to
the exact optimization model for the various optimization algorithms
we should forewarn the readers that each of these algorithms optimize
potentially different objectives and therefore comparing the values of
$\delta_t$ directly would be misrepresentative. The main purpose of
the approximate analysis is to illustrate that there exist precision
accuracies so that if an optimization oracle optimizes the objectives
at every iteration to such precision, the resulting regret bounds are
of the same order as the theoretical exact computation setting.





\section{Conclusion}
\label{sec:conclusion}
Recent research~\cite{RossB11,PogVoiRos11} has sought to establish
connections between stability and online learnability. In the light of
our work, it becomes evident that online stability is a crucial
concept in online learning. It is not only related to the ability to
minimize regret but also provides us with a straightforward recipe to
analyze regret for most existing online learning algorithms via a
remarkably simplified analysis.


It will be interesting to see to what extent this result extends to
arbitrary non-convex sets. Finally, stability based proofs for regret
bounds of algorithms such as FTRL, IOL and RDA easily extend to the
case where the optimization problem arising at every step of these
algorithms is only solved approximately. This opens up many avenues
for further exploration. Can we compare algorithms based on the
trade-offs they offer between low regret and small amount of
computation per step?  Like regularization and random perturbations,
can approximate computation itself serve as the source of stability in
online learning algorithms?

In contrast to the iid setting, there is unfortunately still a
significant gap in our understanding the role of stability for online
learning. The biggest shortcoming of existing work is that most of the
stability based analysis (including ours) in online learning is still
based on analyzing stability of \emph{algorithms}.  A connection of
stability with the online learnability of the underlying concept class
is still missing.
In contrast, \cite{ShaShaSreSri10} provides a generic equivalence
between the existence of a stable AERM and the learnability of a
concept class in the generic batch setting. We think that a major
reason behind this is the absence of a canonical scheme like Empirical
Risk Minimization which can
characterize online learnability for all concept classes.  While our
definition of online stability provides a new way of looking at
online regret, it is still an open problem to understand stability
and online learnability \cite{RakhlinST10} fundamentally in a manner
akin to the batch learning framework.


\clearpage
\bibliographystyle{plain}
{\small
\bibliography{refs}}

\newpage

\end{document}